\documentclass[nohyperref]{article}

\usepackage{microtype}
\usepackage{graphicx}
\usepackage{subcaption}
\usepackage{booktabs} 
\usepackage{siunitx}
\usepackage{hyperref}

\usepackage[accepted]{styles/icml2022}

\usepackage{amsmath}
\usepackage{amssymb}
\usepackage{mathtools}
\usepackage{amsthm}
\usepackage{styles/symbols}
\usepackage{comment}
\usepackage{wrapfig}

\usepackage[capitalize,noabbrev]{cleveref}
\theoremstyle{plain}
\newtheorem{theorem}{Theorem}[section]
\newtheorem{proposition}[theorem]{Proposition}
\newtheorem{lemma}[theorem]{Lemma}
\newtheorem{corollary}[theorem]{Corollary}
\theoremstyle{definition}

\newtheorem{assumption}[theorem]{Assumption}
\theoremstyle{remark}

\crefname{theorem}{Theorem}{Theorems}
\crefname{assumption}{Assumption}{Assumptions}

\usepackage[textsize=tiny]{todonotes}

\icmltitlerunning{Importance Weighted Kernel Bayes' Rule}

\AtBeginDocument{
  \abovedisplayskip     =0.5\abovedisplayskip
  \abovedisplayshortskip=0.5\abovedisplayshortskip
  \belowdisplayskip     =0.5\belowdisplayskip
  \belowdisplayshortskip=0.5\belowdisplayshortskip
 }

\renewcommand{\algorithmicrequire}{\textbf{Input:}}
\renewcommand{\algorithmicensure}{\textbf{Output:}}

\begin{document}

\twocolumn[
\icmltitle{Importance Weighting Approach in Kernel Bayes' Rule}



\icmlsetsymbol{equal}{*}

\begin{icmlauthorlist}
\icmlauthor{Liyuan Xu}{g}
\icmlauthor{Yutian Chen}{d}
\icmlauthor{Arnaud Doucet}{d}
\icmlauthor{Arthur Gretton}{g}

\end{icmlauthorlist}

\icmlaffiliation{g}{Gatsby Unit}
\icmlaffiliation{d}{DeepMind}

\icmlcorrespondingauthor{Liyuan Xu}{liyuan.jo.19@ucl.ac.uk}

\icmlkeywords{Kernel Methods, Approximate Bayesian Inference}

\vskip 0.3in
]
\printAffiliationsAndNotice{}

\begin{abstract}

We study a nonparametric approach to Bayesian computation via feature means, where the expectation of prior features is updated to yield expected kernel posterior features, based on regression from learned neural net or  kernel  features of the observations. All quantities involved in the Bayesian update are learned from observed data, making the method entirely model-free. The resulting algorithm is a novel instance of a kernel Bayes' rule (KBR),  based on importance weighting. This results in superior numerical stability to the original approach to KBR, which requires operator inversion.  We show the convergence of the estimator using a novel consistency analysis on the importance weighting estimator in the infinity norm. We evaluate  KBR on challenging synthetic benchmarks, including a filtering problem with a state-space model involving high dimensional image observations. Importance weighted KBR  yields uniformly better empirical performance than the original KBR, and competitive performance with other competing methods.

\end{abstract}
\section{Introduction} \label{sec:introduction}
Many machine learning applications are reduced to the problem of inferring latent variables in probabilistic models.
This is achieved using Bayes' rule, where a prior distribution over the latent variables is updated to obtain the posterior distribution, based on a likelihood function.  Probabilities involved in the Bayes' update have generally been expressed as probability density functions.
For many interesting problems, the exact posterior density is intractable; in the event that the likelihood is known,
we may use approximate inference techniques, including \emph{Markov Chain Monte Carlo (MCMC)} and \emph{Variational inference (VI)}.
When the likelihood function is unknown, however, and only samples drawn from it can be obtained, these methods do not apply.


We propose to use a {\em kernel mean embedding} representation of our probability distributions as the key quantity in our Bayesian updates \citep{SmoGreSonSch07},
which will enable nonparametric inference without likelihood functions.
Kernel mean embeddings characterize probability distributions as expectations of features in a reproducing kernel Hilbert space (RKHS). When the space is {\em characteristic},
these embeddings are injective \citep{Fukumizu2009KME,Sriperumbudur2010KME}. Expectations of RKHS functions may be expressed as dot products with the corresponding mean embeddings. Kernel mean embeddings have been employed extensively in nonparametric hypothesis testing \citep{Gretton2012MMD,ChwStrGre16}, however they have also been used in supervised learning for distribution-valued inputs
\citep{MuaFukDinSch12,SzaSriPocGre16}, and in various other inference settings \citep{Song2009CME,GruLevBalPatetal12,Boots2013PSR,JitGreHeeElsetal15,Singh2019KIV,MuaKanSaeMar21}.

We will focus here on the  \emph{Kernel Bayes' Rule (KBR)}, where a prior mean embedding is taken as input and is updated to return the mean embedding of the posterior. The goal is to express all quantities involved in the KBR updates as RKHS functions learned from observed data:  parametric models should not be required at any stage in the computation. We will address in particular the application of kernel Bayes' rule to filtering, where a latent state evolves stochastically over time according to Markovian dynamics, with noisy, high dimensional observations providing information as to the current state: the task is to construct a posterior distribution over the current state from the sequence of noisy  observations up to the present. In the event that a parametric model is known for the latent dynamics, then filtering can be achieved either by sampling from the model \cite{KanNisGreFuk16}, or through closed-form updates \cite{NisKanGreFuk20}, however, we propose here to make no modeling assumptions concerning the latent dynamics, beyond observing them in a training stage.   An alternative to kernel Bayes updates in the filtering setting was proposed by \citet{jasra2012filtering} but, for high-dimensional observations, it requires introducing summary statistics which have a high impact on performance and are difficult to select.

The first formulation of Kernel Bayes' Rule, due to  \citet{Fukumizu2013KBR}, yields the desired model-free, sample-derived posterior embedding estimates, and a filtering procedure that is likewise nonparametric and model-free.   At a high level, the resulting posterior is obtained via a two-stage least squares procedure, where the output of the first regression stage
is used as input to the second stage. Unfortunately, the original KBR formulation requires an unconventional form of regularization in the second stage, which
adversely impacts the attainable convergence rates (see Section \ref{subsec:convergence-analysis}), and decreases performance in practice. Following this work, \citet{Yang2016KBR2} proposed ``thresholding regularization,'' which avoids the need for this problematic regularization by thresholding the weights of first regression stage at zero. Although this regularization exhibits superior performance in the empirical evaluations, the theoretical analysis is confined to the case when the latent is discrete, and the convergence rate is not explicitly derived.

In the present work, we introduce {\em importance-weighted KBR (IW-KBR)}, a novel design for a KBR which does not require the problematic second-stage regularization.  The core idea of our approach is to use importance weighting, rather than two-stage regression, to achieve the required Bayesian update. This generalizes the ``thresholding regularization'' \citep{Yang2016KBR2}, which can be interpreted as IW-KBR with the importance weights estimated by KuLSIF estimator \citep{Kanamori2012KuLSIF}.
 We provide a  convergence analysis in a general probability space, requiring as an intermediate step a novel analysis on 
 KuLSIF estimator \citep{Kanamori2012KuLSIF} in infinity norm, which may be of independence interest.  
 Our IW-KBR improves on the convergence rate of the original KBR under certain conditions. As an algorithmic contribution, we introduce adaptive  neural network features into IW-KBR, which is essential when the observations are images or other  complex objects. In experiments, IW-KBR outperforms the original KBR  across all benchmarks considered, including filtering problems with state-space models involving high dimensional image observations.

The paper is structured as follows. In \cref{sec:preliminary}, we introduce the basic concepts of kernel methods and review the original KBR approach. We will also introduce density ratio estimators, which are used in estimating the importance weights.  Then, in \cref{sec:proposed-method} we introduce the proposed KBR approach using kernel and neural net features and provide convergence guarantees when kernel features are employed. We describe the \emph{Kernel Bayes Filter (KBF)} in \cref{sec:kbf}, which applies KBR to the filtering problem in a state-space model. We demonstrate the empirical performance in \cref{sec:experiment}, covering three settings: non-parametric Bayesian inference,  low-dimensional synthetic settings introduced in \citet{Fukumizu2013KBR}, and challenging KBF settings where the observations consist of high-dimensional images. 

\section{Background and preliminaries} \label{sec:preliminary}
In this section, we introduce kernel mean embeddings, which represent probability distributions by  expected RKHS features. We then review Kernel Bayes' Rule (KBR) \citep{Fukumizu2013KBR}, which aims to learn a mean posterior embedding according to Bayes' rule.


{\bf Kernel mean embeddings:} Let $(X,Z)$ be  random variables on $\mathcal{X} \times \mathcal{Z}$ with distribution $P$ and density function $p(x,z)$, and $k_\calX: \calX \times \calX \to \mathbb{R}$ and $k_\calZ: \calZ \times \calZ \to \mathbb{R}$ be measurable positive definite kernels corresponding to the scalar-valued RKHSs $\calH_\calX$ and $\calH_\calZ$, respectively. We denote the feature maps as $\psi(x) = k_\calX(x, \cdot)$ and $\phi(z) = k_\calZ(z, \cdot)$, and RKHS norm as $\|\cdot\|$.

The kernel mean embedding $m_{P(X)}$ of a marginal distribution $P(X)$ is defined as 
\begin{align*}
    m_{P(X)} = \expect[P]{\psi(X)} \in \calH_\calX,
\end{align*}
and always exists for bounded kernels. From the reproducing property, we have $\braket{f, m_{P(X)}} = \expect[P]{f(X)}$ for all $f\in\calH_\calX$, which is useful to estimate the expectations of functions. Furthermore, it is known that the embedding uniquely defines the probability distribution (i.e. $m_{P(X)} = m_{Q(X)}$ implies $P=Q$) if kernel $k_\calX$ is \emph{characteristic}: for instance, the Gaussian kernel has this property \cite{Fukumizu2009KME,Sriperumbudur2010KME}.
In addition to kernel means, we will require kernel covariance operators, 
\begin{alignat*}{3}
    &\covPXX\!=\!\expect[P]{\psi(X)\!\otimes\!\psi(X)}, &~~~& \covPXZ \!=\! \expect[P]{\psi(X)\!\otimes\!\phi(Z)},\\
    &\covPZX\!=\!\left(\covPXZ\right)^*, &&\covPZZ =  \expect[P]{\phi(Z)\!\otimes\!\phi(Z)},
\end{alignat*}
 where $\otimes$ is the tensor product such that $[a\otimes b]c= a \braket{b, c},$   and ${}^*$ denotes the adjoint of the operator.
Covariance operators generalize finite-dimensional covariance matrices to
the case of infinite kernel feature spaces, and always exist for bounded kernels.

The \emph{kernel conditional mean embedding} \citep{Song2009CME} is the extension of the kernel mean embedding to conditional probability distributions, and is defined as 
\begin{align*}
    m_{P(Z|X)}(x) = \expect[P]{\phi(Z)|X=x} \in \calH_\calZ.
\end{align*}
 Under the regularity condition $\expect[P]{g(Z)|X=x}\in \calH_\calX$ for all $g\in\calH_\calZ,$ there exists an operator $\EP \in \calH_\calX \otimes \calH_\calZ$ such that $m_{P(Z|X)}(x) = \left(\EP\right)^* \psi(x)$ \citep{Song2009CME,GruLevBalPatetal12,Singh2019KIV}.
The   \emph{conditional operator} $\EP$  can be expressed as the minimizer of the surrogate regression loss $\lossP$, defined as
\begin{align*}
    \lossP(E) = \expect[P]{\|\phi(Z) - E^* \psi(X)\|^2}.
\end{align*}
This minimization can be solved analytically, and the closed-form solution is given as
\begin{align*}
    \EP = (\covPXX)^{-1}\covPXZ.
\end{align*}

An empirical estimate of the conditional mean embedding is straightforward. Given i.i.d samples $\{(x_i, z_i)\}_{i=1}^n \sim P$, we minimize $\hatlossP,$ defined as 
\begin{align*}
    \hatlossP(E) = \frac1n \sum_{i=1}^n \|\phi(z_i) - E^* \psi(x_i)\|^2 + \lambda \|E\|^2,
\end{align*}
which is the sample estimate of $\lossP$  with added Tikhonov regularization, and the norm $\|E\|$ is the Hilbert-Schmidt norm. The solution of this minimization problem is
\begin{align*}
   \hatregEP =  (\hatcovPXX + \lambda I)^{-1}\hatcovPXZ. 
\end{align*}
Here, we denote $\hatcovPZX, \hatcovPXX, \hatcovPXZ, \hatcovPZZ$ as the empirical estimates of the covariance operators $\covPZX, \covPXX, \covPXZ, \covPZZ$, respectively; e.g. $\hatcovPZX = \frac1n \sum_{i=1}^n \phi(z_i) \otimes \psi(x_i)$.


{\bf Kernel Bayes' Rule:} Based on conditional mean embedding, \citet{Fukumizu2013KBR} proposed a method to realize Bayes' rule in kernel mean embedding representation. In Bayes' rule, we aim to update the prior distribution $\Pi$ on latent variables $Z$, with the density function $\pi(z)$, to the posterior $Q(Z|X=\tilde x)$, where the joint density $q(z,x)$ of distribution $Q$ is given by \begin{align*}
    q(z,x) =\pi(z)g(x|z).
\end{align*}
Here, $g(x|z)$ is the \emph{likelihood} function and $\tilde x \in \calX$ is the \emph{conditioning point}. The density of $Q(Z|X=\tilde x)$ is often intractable, since it involves computing the  integral $\int q(z,x)\intd z$. Instead, KBR aims to update the embedding of $\Pi$, denoted as $m_\Pi$, to the embedding of the posterior $m_{Q(Z|X)}(\tilde x)$. \citet{Fukumizu2013KBR} show that such an update does not require the closed-form expression of likelihood function $g(x|z)$. Rather, KBR \emph{learns} the relations between latent and observable variables from the data $\{(x_i, z_i)\}_{i=1}^n \sim P,$ where the density of data distribution $P$ shares the same likelihood $p(z,x)=p(z)g(x|z)$. We require $P$ to share the  likelihood of $Q$, otherwise the relation between observations and latent cannot be learned.
We remark that KBR also applies to \emph{Approximate Bayesian Computation (ABC)} \citep{ABC1,ABC2}, in which the likelihood function is intractable, but we can simulate from it.

If $\pi(z) = p(z)$, we could simply estimate $m_{Q(Z|X)}(\tilde x)$ by the conditional mean embedding learned from  $\{(x_i, z_i)\}_{i=1}^n$, but we will generally require $\pi(z) \neq p(z),$ since the prior can be the result of another inference step (notably for filtering: see \cref{sec:kbf}). We
present the solution of \citet{Fukumizu2013KBR}  for this case.
Let $\covQZX, \covQXX, \covQZX, \covQXX$ be the covariance operators on distribution $Q$, defined similarly to the  covariance operators on $P$. Similarly, the conditional operator $\EQ: \calH_\calX\otimes\calH_\calZ$ satisfying $m_{Q(Z|X)}(x) = \EQ^* \psi(x)$ is a minimizer of the loss 
\begin{align*}
    \lossQ(E) = \expect[Q]{\|\phi(Z) - E^* \psi(X)\|^2}.
\end{align*}
Unlike in the conditional mean embedding case, however, we cannot directly minimize this loss, since data $\{x_i, z_i\}$ are not sampled from $Q$ . Instead, \citeauthor{Fukumizu2013KBR} uses the analytical form of $\EQ$:
\begin{align*}
    \EQ = \covQZX (\covQXX)^{-1},
\end{align*}
and replaces each operator with vector-valued kernel ridge regression estimates,
\begin{align*}
   &\hatcovorgQXZ\!=\!\sum_{i=1}^n \gamma_i\psi(x_i)\!\otimes\!\phi(z_i),\\
   &\hatcovorgQXX\!=\!\sum_{i=1}^n\gamma_i \psi(x_i)\!\otimes\!\psi(x_i),
\end{align*}
where each weight $\gamma_i$ is given as 
\begin{align}
    \gamma_i = \braket{\phi(z_i), \left(\hatcovPZZ + \eta I\right)^{-1}\hat{m}_\Pi}. \label{eq:original-gamma}
\end{align}
Here, $\eta$ is another Tikhonov regularization parameter. Although these estimators are consistent, $\hatcovorgQXX$ is not necessarily positive semi-definite since weight $\gamma_i$ can be negative. This causes instabilities when inverting the operator. \citeauthor{Fukumizu2013KBR} mitigate this by applying another type of Tikhonov regularization, yielding an alternative estimate of $\EQ$,
\begin{align}
     \hatregorgEQ = \hatcovorgQXX\left(\left(\hatcovorgQXX\right)^{2} + \lambda I\right)^{-1}\hatcovorgQXZ. \label{eq:original-kbr}
\end{align}



{\bf Density Ratio Estimation:} The core idea of our proposed approach is to use importance sampling to estimate $\lossQ$. To obtain the weights, we need to estimate the density ratio 
\begin{align*}
    r_0(z) = \pi(z) / p(z).
\end{align*}
We may use any density ratio estimator, 
as long as it is computable from data $\{z_i\}_{i=1}^n$ and the prior embedding estimator $\hat{m}_\Pi$. We focus here on the \emph{Kernel-Based unconstrained Least-Squares Importance Fitting (KuLSIF)} estimator \citep{Kanamori2012KuLSIF}, which is obtained by minimizing
\begin{align*}
    \tilde{r}_\eta = \argmin_{r\in\calH_\calZ} \frac12 \braket{r, \hatcovPZZ r} - \braket{r, \hat{m}_\Pi} + \frac{\eta}2\|r\|^2.
\end{align*}
The estimator $\hat{r}(z)$ is obtained by truncating the solution $\tilde r$ at zero: $\hat{r}(z) = \max(0, \tilde{r}_\eta(z))$. Interestingly, the KuLSIF estimator at data point $z_i$ can be written
\begin{align*}
    \hat{r}(z_i) = \max(0, \gamma_i),
\end{align*}
where $\gamma_i$ is the weight used in \eqref{eq:original-gamma}. \citet{Kanamori2012KuLSIF} developed a convergence analysis of this KuLSIF estimator in $\ell_2$-norm based on the bracketing entropy \citep{Cucker2001OnTM}. This analysis, however, is insufficient for the KBR case, and we will establish stronger convergence results under different assumptions in \cref{subsec:convergence-analysis}.

\section{Importance-Weighted KBR} \label{sec:proposed-method}
In this section, we introduce our \emph{importance-weighted KBR (IW-KBR)} approach and provide a convergence analysis. We also propose a method to learn adaptive features using neural networks so that the model can learn complex posterior distributions.

\subsection{Importance Weighted KBR} \label{subsec:IWKBR}
Our proposed method minimizes the loss $\lossQ$, which is estimated by the importance sampling. Using density ratio $r_0(z) = \pi(z) / p(z)$, the loss $\lossQ$ can be rewritten as 
\begin{align*}
    \lossQ(E) = \expect[P]{r_0(Z) \|\phi(Z) - E^*\psi(X)\|^2}.
\end{align*}
Hence, we can construct the empirical loss  with added Tikhonov regularization,
\begin{align}
    \hatlossQ(E) = \frac1n \sum_{i=1}^n \hat{r}_i  \|\phi(z_i) - E^*\psi(x_i)\|^2 + \lambda \|E\|^2, \label{eq:hatlossQ}
\end{align}
where $\hat{r}_i = \hat{r}(z_i)$ is a non-negative estimator of density ratio $r_0(z_i)$. Again, the minimizer of $\hatlossQ(E)$ can be obtained analytically as 
\begin{align}
     \hatregEQ = \left(\hatcovQXX + \lambda I\right)^{-1}\hatcovQXZ, \label{eq:proposed-kbr}
\end{align}
where
\begin{align*}
   \hatcovQXZ\!=\!\sum_{i=1}^n \frac{\hat{r}_i}{n} \psi(x_i)\!\otimes\!\phi(z_i), \hatcovQXX\!=\!\sum_{i=1}^n \frac{\hat{r}_i}{n} \psi(x_i)\!\otimes\!\psi(x_i).
\end{align*}
Note that $\hatcovQXX$ is always positive semi-definite since $\hat{r}_i$ is non-negative by definition. Using the KuLSIF estimator, this is the truncated weight $\hat{r}_i = \max(0, \gamma_i)$ described previously, which considers the same empirical estimator $\hatcovQXX$ as ``thresholding regularization'' \citep{Yang2016KBR2}.


Given estimated conditional operator $\hatregEQ$ in \eqref{eq:proposed-kbr}, we can estimate the conditional embedding $m_{Q(Z|X)}(x)$ as $\hat{m}_{Q(Z|X)}(x) = (\hatregEQ)^*\psi(x)$, as shown in \cref{alg:kbr}.  As illustrated, the posterior mean embedding $\hat{m}_{Q(Z|X)}$ is represented by the weighted sum over the same RKHS features $\phi(z) = k_\calZ(z, \cdot)$ used in the density ratio estimator $\hat{\vec{r}}$.
We remark that this need not be the case, however: the weights $\vec{w}$ could be used to obtain a posterior mean embedding over a {\em different} feature
space to that used in computing $\hat{\vec{r}}$, simply by substituting the desired $\tilde{\phi}(z_i)$ for $\phi(z_i)$ in the sum from the third step.
For example, we could use a Gaussian kernel to compute the density ratio $\hat{\vec{r}}$, and then a linear kernel $\tilde\phi(z) = z$ to estimate the posterior mean   $\hat{\mathbb{E}}_{Q}[Z|X]= \sum_i w_iz_i$.


The computational complexity of \cref{alg:kbr} is $O(n^3)$, which is the same complexity as the ordinary kernel ridge regression. This can be accelerated by using Random Fourier Features \citep{Rahimi2007RFF} or the  Nystr\"om Method \citep{Williams2001Nystrom}.

\begin{algorithm}[tb]
   \caption{Importance Weighted Kernel Bayes Rule}
   \label{alg:kbr}
\begin{algorithmic}[1]
\setlength\belowdisplayskip{2pt}
\setlength\abovedisplayskip{2pt}
   \REQUIRE Samples $\{(x_i, z_i)\}_{i=1}^n \sim P$, Estimated prior embedding $\hat{m}_\Pi$, regularization parameters $(\eta, \lambda)$, Conditioning point $\tilde x$.
   \STATE Compute Gram matrices $G_X, G_Z \in \mathbb{R}^{n\times n}.$
   \begin{align*} 
       (G_X)_{ij} = k_\calX(x_i, x_j), (G_Z)_{ij} = k_\calZ(z_i, z_j).
   \end{align*}
   \STATE Compute KuLSIF $\hat{\vec{r}} = (\hat{r}_1, \dots, \hat{r}_n) \in \mathbb{R}^n$ as 
   \begin{align*} 
       \hat{\vec{r}} = \max\left(0, n\left(G_Z + n\eta I\right)^{-1} \vec{g}_\Pi\right),
   \end{align*}
   where $\max$ operates element-wisely and 
   \begin{align*}
       (\vec{g}_\Pi)_i = \braket{\hat{m}_\Pi, \phi(z_i)}.
   \end{align*} 
   \STATE Estimate $\hat{m}_{Q(Z|X)}(x) = \sum_{i=1}^n w_i\phi(z_i)$, with weight $\vec{w} = (w_1, \dots, w_n)$ given as 
   \begin{align*}
       \vec{w} = \sqrt{D}\left(\sqrt{D}G_X\sqrt{D} + n\lambda I\right)^{-1} \sqrt{D}\vec{g}_{\tilde x},
   \end{align*}
   where $D = \mathrm{diag}(\hat{\vec{r}}) \in \mathbb{R}^{n\times n}$ and $(\vec{g}_{\tilde x})_i = k_\calX(x_i, \tilde x)$. 
\end{algorithmic}
\end{algorithm}

\subsection{Convergence Analysis} \label{subsec:convergence-analysis}
In this section, we analyze $\hatregEQ$ in \eqref{eq:proposed-kbr}. First, we state a few regularity assumptions.
\begin{assumption} \label{assum:bound-feature-norm}
    The kernels $k_\calX$ and $k_\calZ$ are continuous and bounded: $\|\phi(z)\| \leq \kappa_1< \infty, \|\psi(x)\| \leq \kappa_2 < \infty$ hold almost surely.
\end{assumption}
\begin{assumption}\label{assum:density-ratio}
    We have $r_0 \in \calH_\calZ$ and $\exists g_1 \in \calH_\calZ$ such that $r_0 = (\covPZZ)^{\beta_1}  g_1$ and $\|g_1\| \leq \zeta_1$ for $\beta_1 \in (0, 1/2], \zeta_1 < \infty$ given.
\end{assumption}
\begin{assumption} \label{assum:covariance-operator}
    $\exists \Gamma \in \calH_\calX \otimes \calH_\calZ $ such that $\EQ = (\covQXX)^{\beta_2} \Gamma$ and $\|\Gamma\| \leq \zeta_2$ for $\beta_2 \in (0, 1/2], \zeta_2 < \infty$ given.
\end{assumption}

\cref{assum:bound-feature-norm} is standard for mean embeddings, and many widely used kernel functions satisfy it, including the Gaussian kernel and the Mat\'{e}rn kernel. \cref{assum:density-ratio,assum:covariance-operator}  assure the smoothness of the density ratio $r_0$ and conditional operator $E_Q$, respectively. See \citep{Smale2007LT,CapDev07} for a discussion of the first assumption, and \citep[Hypothesis 5]{Singh2019KIV} for the second. We will further discuss the implication of \cref{assum:density-ratio} using the ratio of two Gaussian distributions in \cref{sec:dre-assume}.
Note that \cref{assum:density-ratio}, also used in  \citet{Fukumizu2013KBR}, should be treated with care. For example, $r_0 \in \calH_\calZ$ imposes $\sup_{z\in\calZ} r_0(z) \leq \|r_0\| \sup_{z\in\calZ} \|\phi(z)\| <\infty$, which is violated when we consider the ratio of two Gaussian distributions with the different means and the same variance. 

We now show that the KuLSIF estimator converges in infinity norm $\|\cdot\|_{\infty}$.
\begin{theorem} \label{thm:stage1-converge}
Suppose \cref{assum:bound-feature-norm,assum:density-ratio}. Given data $\{x_i, z_i\} \sim P$ and the estimated prior embedding $\hat{m}_\Pi$ such that $\|\hat{m}_\Pi - \expect[\Pi]{\phi(X)}\|\leq O_p(n^{-\alpha_1})$ for $\alpha_1 \in (0, 1/2]$, by setting $\eta = O(n^{-\frac{\alpha_1}{\beta_1 + 1}})$, we have
\begin{align*}
    \|r_0 - \hat{r}\|_{\infty} \leq O_p\left(n^{-\frac{\alpha_1 \beta_1}{\beta_1 + 1}}\right).
\end{align*}
\end{theorem}
The proof is given in \cref{subsec:proof}. This result differs from the analysis in \citet{Kanamori2012KuLSIF}, which establishes convergence of the KuLSIF estimator in $\ell_2$-norm, based on the bracketing entropy \citep{Cucker2001OnTM}.
Our result cannot be directly compared with theirs, due to the differences in the assumptions made and the norms used in measuring convergence (in particular, we
require convergence in infinity norm).
Establishing a relation between the two results represents an interesting research direction, though it is out of the scope of this paper.

\cref{thm:stage1-converge} can then be used to obtain the convergence rate of covariance operators.
\begin{corollary} \label{thm:stage1-converge-operator}
Given the same conditions in \cref{thm:stage1-converge},  by setting $\eta = O(n^{-\frac{\alpha_1}{\beta_1 + 1}})$, we have
\begin{align*}
    \|\hatcovQXX - \covQXX\|  \leq O_p\left(n^{-\frac{\alpha_1 \beta_1}{\beta_1 + 1}}\right).
\end{align*}
\end{corollary}
 The  cross covariance operator $\hatcovQXZ$ also converges at the same rate.
This rate is slower than the original KBR estimator, however, which satisfies
\begin{align*}
    \|\hatcovorgQXX - \covQXX\| \leq O_p\left(n^{-\frac{\alpha_1 (2\beta_1 + 1)}{2\beta_1 + 2}}\right)
\end{align*}
 \citep{Fukumizu2013KBR}. This is inevitable since the original KBR estimator uses the optimal weights $\gamma_i$ to estimate $\covQXX$, while our estimator uses truncated weights $\hat{r}_i = \max(0, \gamma_i)$, which introduces a bias in the estimation.

Given our consistency result on  $\hatcovQXX, \hatcovQXZ$, we can show that the estimated conditional operators are also consistent.
\begin{theorem} \label{thm:stage2-converge}
Suppose \cref{assum:bound-feature-norm,assum:covariance-operator}. Given data $\{x_i, z_i\} \sim P$ and estimated covariance operators such that $\|\hatcovQXX - \covQXX\| \leq O_p(n^{-\alpha_2})$ and 
$\|\hatcovQXZ - \covQXZ\| \leq O_p(n^{-\alpha_2})$, by setting $\lambda = O(n^{-\frac{\alpha_2}{\beta_2 + 1}})$ we have 
\begin{align*}
    \|\hatregEQ - \EQ\| \leq O_p\left(n^{-\frac{\alpha_2\beta_2}{\beta_2 + 1}}\right).
\end{align*}
\end{theorem}
The proof is  in \cref{subsec:proof}. This rate is faster than the original KBR estimator \citep{Fukumizu2013KBR}, which satisfies
\begin{align*}
    \|\hatregorgEQ - \EQ\| \leq O_p\left(n^{-\frac{2\alpha_2\beta_2}{2\beta_2 + 5}}\right).
\end{align*}
This is the benefit of avoiding the regularization of the form \eqref{eq:original-kbr} and using instead \eqref{eq:proposed-kbr}. Given \cref{thm:stage1-converge-operator,thm:stage2-converge}, we can thus show that
\begin{align*}
    \|\hatregEQ - \EQ\| \leq O_p\left(n^{-\frac{\alpha_1\beta_1\beta_2}{(\beta_1 + 1)(\beta_2+1)}}\right),
\end{align*}
while \citet{Fukumizu2013KBR} obtained
\begin{align*}
    \|\hatregorgEQ -  \EQ\| \leq O_p\left(n^{-\frac{\alpha_1(2\beta_1+1)2\beta_2}{(2\beta_1 + 2)(2\beta_2+5)}}\right).
\end{align*}
The approach that yields the better overall rate depends on the smoothness parameters $(\beta_1, \beta_2)$. Our approach converges faster than the original KBR when the density ratio $r_0$ is smooth (i.e. $\beta_1 \simeq 1/2$) and the conditional operator  less smooth (i.e. $\beta_2 \simeq 0$).
Note further that  \citet{Sugiyama2008} show, even when $r_0 \neq \calH_\calZ$, that the KuLSIF estimator $\hat{r}$ converges to the element in RKHS $\calH_\calZ$ with the least $\ell_2$ error. We conjecture that our method might thus be robust to misspecification of the density estimator,
although this remains a topic for future research (in particular, our proof requires consistency of the form in \cref{thm:stage1-converge}).

\subsection{Learning Adaptive Features in KBR} \label{subsec:adaptive-feature}

Although kernel methods benefit from strong theoretical guarantees, a restriction to RKHS features limits our scope and flexibility, since this requires pre-specified  feature maps. Empirically, poor performance can result in cases where the  observable  variables are high-dimensional (e.g. images), or have highly nonlinear relationships with the latents. Learned, adaptive neural network features have previously been used to substitute for kernel features when performing inference on mean embeddings in causal modeling \citep{xu2021learning,Xu2021Proxy}. Inspired by this work, we propose to employ adaptive NN observation features in KBR.

Recall that we learn the conditional operator $\hatregEQ$ by minimizing the loss $\hatlossQ$ defined in \eqref{eq:hatlossQ}. We propose to jointly learn the feature map $\psi_{\theta}$ with the conditional operator $\hatregEQ$ by minimizing the same loss,
\begin{align*}
    \hatlossQ(E, \theta) = \frac1n \sum_{i=1}^n \hat{r}_i  \|\phi(z_i) - E^*\psi_\theta(x_i)\|^2 + \lambda \|E\|^2,
\end{align*}
where $\psi_\theta: \calX \to \mathbb{R}^{d}$ is the $d$-dimensional adaptive feature represented by a neural network parameterized by $\theta$. As in the kernel feature case, the optimal operator can be obtained from  \eqref{eq:proposed-kbr} for a given value of $\theta$. From this, we can write the loss for $\theta$ as
\begin{align*}
    \ell(\theta) &= \argmin_E \hatlossQ(E, \theta) \\
    &= \mathrm{tr}\left(G_Z(D - D \Psi^\top_\theta(\Psi_\theta D \Psi_\theta^\top + \lambda I)^{-1}\Psi_\theta D)\right),
\end{align*}
where $G_Z$ is the Gram matrix $(G_Z)_{ij} = k_\calZ(z_i, z_j)$, $D$ is the diagonal matrix with estimated importance weights $D = \mathrm{diag}(\hat{r}_1, \dots, \hat{r}_n)$ and $\Psi_\theta \in \mathbb{R}^{n\times d}$ is the feature matrix defined as $\Psi_\theta = [\psi_\theta(x_1), \dots,  \psi_\theta(x_n)]$. See \cref{sec:prof} for the derivation. 

The loss $\ell(\theta)$ can be minimized using gradient based optimization methods. Given the learned parameter $\hat{\theta} = \argmin \ell(\theta)$ and the conditioning point $\tilde{x}$, we can estimate the posterior embedding $\hat{m}_{Q(Z|X)}(x) = \sum_{i=1}^n w_i\phi(z_i)$ with weights $\vec{w} = (w_1, \dots, w_n)$ given by
\begin{align*}
    \vec{w} = D \Psi^\top_{\hat\theta}\left(\Psi_{\hat\theta} D \Psi^\top_{\hat\theta} + \lambda I \right)^{-1} \psi_{\hat\theta}(\tilde x).
\end{align*}
Note that this corresponds to using a linear kernel on the learned features, $k_\calX(x, x') = \psi_{\hat\theta}(x)^\top\psi_{\hat\theta}(x'),$ in \cref{alg:kbr}.
In experiments, we further employ a finite-dimensional random Fourier feature approximation of $\phi(z)$ to speed computation  \citep{Rahimi2007RFF}.

\section{Kernel Bayes Filter} \label{sec:kbf}
We next describe an important use-case for KBR, namely the \emph{Kernel Bayes Filter (KBF)} \citep{Fukumizu2013KBR}. Consider the following time invariant state-space model with observable variables $X_t$ and hidden state variables $Z_t$,
\begin{align*}
    p(x_{1:T}, z_{1:T})\!=\!p(x_1|z_1) p(z_1) \prod_{t=2}^T p(x_t|z_t)p(z_t|z_{t-1}).
\end{align*}
Here, $X_{s:t}$ denotes $\{X_s, \dots X_{t}\}$.  Given this state-space model,  the filtering problem aims to infer sequentially the distributions $p(z_t|x_1, \dots x_t)$. Classically, filtering is solved by the Kalman filter, one of its nonlinear extensions (the extended Kalman filter (EKF) or unscented Kalman filter (UKF)), or a particle filter \citep{sarkka2013bayesian}. These methods require knowledge of $p(x_t|z_t)$ and $p(z_t|z_{t-1})$, however. KBF does not require knowing these distributions and learns them from samples $(X_{1:T}, Z_{1:T})$ of both observable and hidden variables in a training phase.

Given a test sequence $\{\tilde x_1, \dots \tilde x_T\}$,  KBF sequentially applies KBR  to obtain kernel embedding ${m}_{Z_t| \tilde x_{1:t}}$, where $m_{Z_t| \tilde x_{1:s}}$ denotes the embedding of the posterior distribution $P(z_t|X_{1:s} = \tilde x_{1:s})$. This can be obtained by iterating the following two steps. Assume that we have the embedding ${m}_{Z_t| \tilde x_{1:t}}$. Then, we can compute the embedding of forward prediction ${m}_{Z_{t+1}| \tilde x_{1:t}}$ by 
\begin{align*}
    {m}_{Z_{t+1}| \tilde x_{1:t}} = (E_{Z_{t+1}|Z_t})^*{m}_{Z_{t}| \tilde x_{1:t}},
\end{align*}
where $E_{Z_{t+1}|Z_t}$ is the conditional operator for $P(z_{t+1}|z_t)$. Empirically, this is estimated from data $\{Z_t, Z_{t+1}\}$,
\begin{align*}
    \hat{E}_{Z_{t+1}|Z_t} = (\hat{C}_{\mathtt{prev}, \mathtt{prev}} + \lambda' I)^{-1} \hat{C}_{\mathtt{prev}, \mathtt{post}},
\end{align*}
where $\lambda'$ is a regularizing coefficient and 
\begin{align*}
    &\hat{C}_{\mathtt{prev}, \mathtt{prev}} = \frac1{T-1} \sum_{t=1}^{T-1} \phi(z_t) \otimes \phi(z_t), \\
    &\hat{C}_{\mathtt{prev}, \mathtt{post}} = \frac1{T-1} \sum_{t=1}^{T-1} \phi(z_t) \otimes \phi(z_{t+1}).
\end{align*}

Given the probability of forward prediction $P(z_{t+1}| x_{1:t})$, we can obtain $P(z_{t+1}| x_{1:t+1})$ by applying  Bayes' rule,
\begin{align*}
    p(z_{t+1}|x_{1:t+1}) \propto p(x_{t+1}|z_{t+1}) p(z_{t+1}|x_{1:t}).
\end{align*}
Hence, we can obtain the filtering embedding ${m}_{Z_{t+1}| \tilde x_{1:t+1}}$   at the next timestep by applying KBR with prior embedding $m_\Pi = {m}_{Z_{t+1}| \tilde x_{1:t}}$ and samples $(X_{1:T}, Z_{1:T})$ at the conditioning point $\tilde x = \tilde x_{t+1}$. By repeating this process, we can conduct filtering for the entire sequence $\tilde x_{1:t}$ .

Empirically, the estimated embedding $\hat{m}_{Z_t| \tilde x_{1:s}}$ is represented by a linear combination of features $\phi(z_t)$ as $\hat{m}_{Z_t| \tilde x_{1:s}} = \sum_{i=1} w_i^{(s,t)} \phi(z_i)$. The update equations for weights $\vec{w}_i^{(s,t)}$ are summarized in \cref{alg:kbf}. If we have any prior knowledge of $P(z_1)$, we can initialize $\vec{w}^{(0,1)}$ accordingly. Otherwise, we can set $\vec{w}^{(0,1)} = (1/T, \dots, 1/T)^\top$, which initialization is used in the experiments.

\section{Related Work on Neural Filtering}

Several recent methods have been proposed combining state-space models with neural networks \cite{Rahul2015DKF,Klushyn2021DeepSSM,Rangapuram2018DeepSSM}, aiming to learn the latent dynamic and the observation models from observed sequences $X_{1:T}$ alone. These approaches assume a {\em parametric} form of the latent dynamics: for example, the Deep Kalman filter \citep{Rahul2015DKF} assumes that the distribution of the latents is Gaussian, with mean and covariance
which are nonlinear functions of the previous latent state. DeepSSM \citep{Rangapuram2018DeepSSM} assumes  linear latent dynamics with Gaussian noise, and EKVAE \citep{Klushyn2021DeepSSM} uses a locally linear Gaussian transition model. These models use the variational inference techniques to learn the parameters, which makes it challenging to  prove the convergence to the true models. In contrast, KBF learns latent dynamics and observation model nonparametrically from samples, and the accuracy of the filtering is guaranteed from the convergence  of KBR.

\begin{algorithm}[tb]
   \caption{Kernel Bayes Filter}
   \label{alg:kbf}
\begin{algorithmic}[1]
\setlength\belowdisplayskip{2pt}
\setlength\abovedisplayskip{2pt}
   \REQUIRE Training Sequence $\{(x_t, z_t)\}_{t=1}^T$, regularizing parameter $(\eta, \lambda, \lambda')$, Test sequence $\{\tilde x_t\}_{t=1}^{\tilde T}$.
   \STATE Initialize $\vec{w}^{(0,1)}$
   \FOR{$t = 1$ \textbf{to} $\tilde T$}
   \STATE $\vec{w}^{(t,t)} \leftarrow$ \cref{alg:kbr} for data $\{(x_t, z_t)\}_{t=1}^T$, prior embedding $\hat{m}_\Pi = \sum_{i=1} w_i^{(t-1,t)} \phi(z_i)$, conditioning point $\tilde x = \tilde x_t$.
   \STATE Compute $\vec{w}^{(t,t+1)}$ with $w_1^{(t,t+1)} = 0$ and 
   \begin{align*}
       \!\vec{w}_{2:T}^{(t,t+1)} = (G_{Z_{-1}} + (T-1)\lambda' I)^{-1} \tilde G_{Z_{-1}} \vec{w}^{(t,t)}
   \end{align*}
   where $(G_{Z_{-1}})_{ij} = k_\calZ(z_i, z_j) \in \mathbb{R}^{T-1 \times T-1}$ and $(\tilde G_{Z_{-1}})_{ij} = k_\calZ(z_i, z_j) \in \mathbb{R}^{T-1 \times T}$.
   \ENDFOR
\end{algorithmic}
\end{algorithm}

\section{Experiments} \label{sec:experiment}
In this section, we empirically investigate the performance of our KBR estimator in a variety of settings, including the problem of learning posterior mean proposed in \citet{Fukumizu2013KBR}, as well as challenging filtering problems where the observations are high-dimensional images.

\subsection{Learning Posterior Mean From Samples}
We revisit here the problem introduced by \citet{Fukumizu2013KBR}, which learns the posterior mean from samples. Let $X \in \mathbb{R}^d$ and $Z \in \mathbb{R}^d$ be generated from $P(X,Z) = \mathcal{N}((\bm{1}_d^\top, \bm{0}_d^\top)^\top, V)$, where  $\bm{1}_d = (1,\dots, 1)^\top \in \mathbb{R}^d, \bm{0}_d = (0,\dots, 0)^\top \in \mathbb{R}^d$ and $V = \frac1{2d}A^\top A + 2I$, with each component of $A \in \mathbb{R}^{2d\times 2d}$  sampled from the standard Gaussian distribution on each run. The prior is set to be $\Pi(Z) = \mathcal{N}(\bm{0}_d, V_{ZZ}/2)$, where $V_{ZZ}$ is the $Z$-component of $V$. We construct the prior embedding using 200 samples from $\Pi(Z)$, and we aim to predict the mean of posterior $\expect[Q]{Z|X}$ using $n = 200$ data points from $P(X,Z)$, as the dimension $d$ varies.

\begin{figure*}[ht]
\begin{minipage}[t]{0.6\columnwidth}
    \centering
    \includegraphics[width=\linewidth]{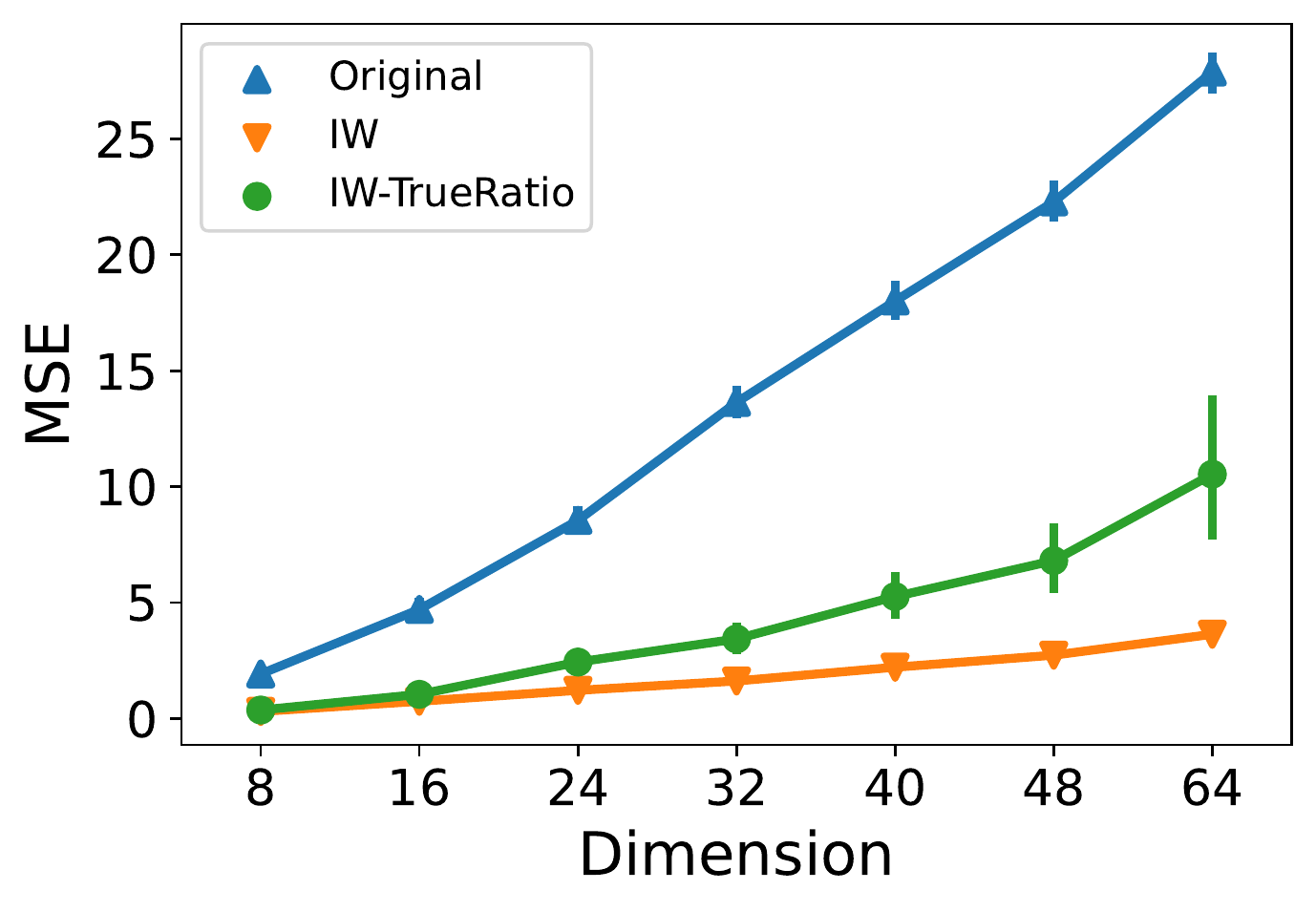}
    \caption{MSE as a function of dimension for posterior mean prediction}
    \label{fig:kbr_res}
  \end{minipage}
  \hfill
  \begin{minipage}[t]{0.98\columnwidth}
   \centering
    \includegraphics[width =\linewidth]{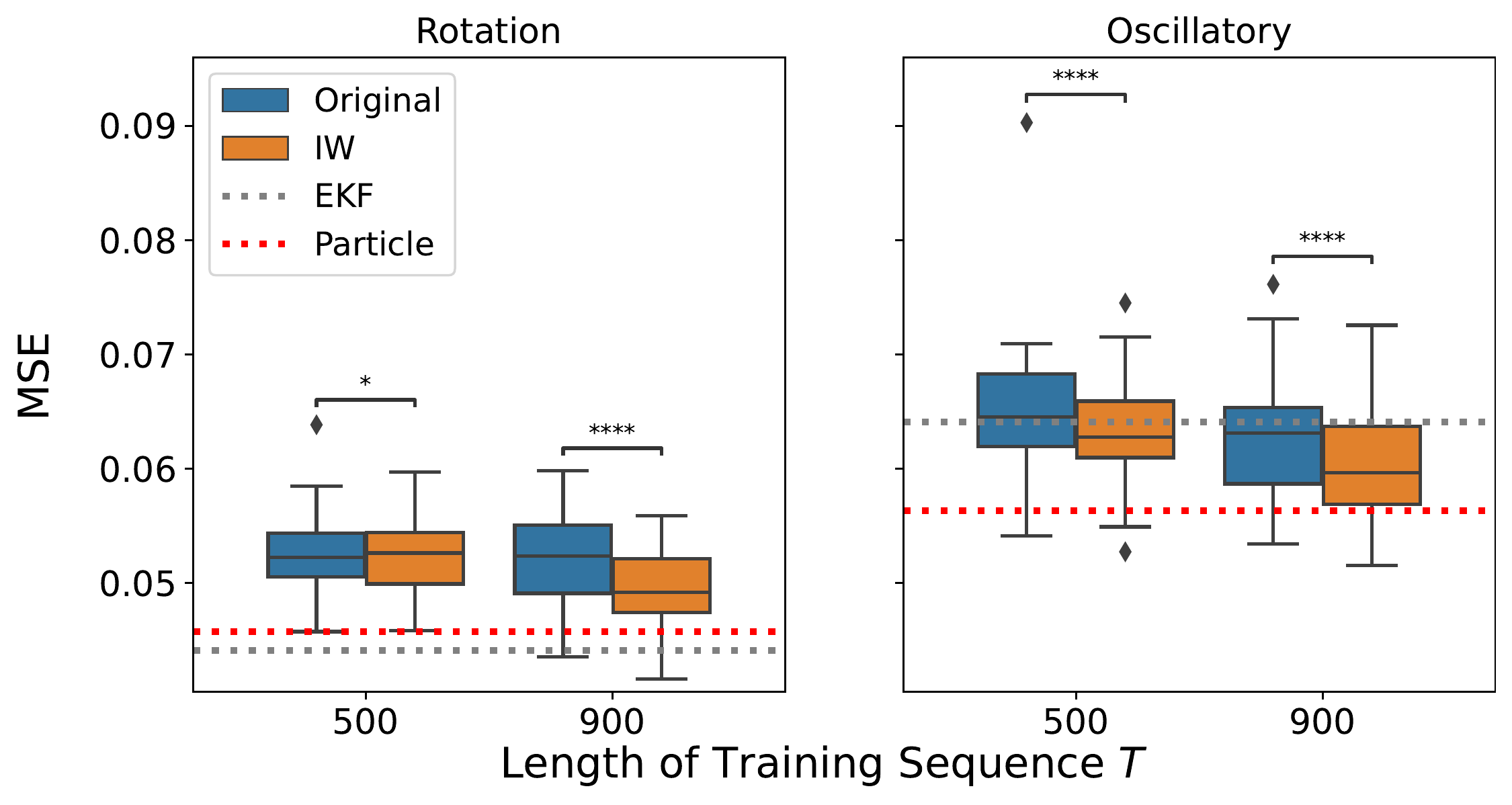}
  \caption{Results for low-dimensional KBF experiments. The asterisks indicates the level of statistical significance (*:  $p< 0.05$, **: $p<0.01$, ***: $p<0.001$, ****: $p< 0.0001$), which is tested by the Wilcoxon signed-rank test.} 
  \label{fig:low-dimensional}
  \end{minipage}
  \hfill
  \begin{minipage}{0.4\columnwidth}
   \centering
   \vspace{-70pt}
   \includegraphics[width=\linewidth]{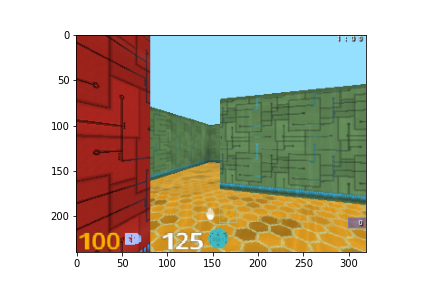}
   \vspace{-20pt}
    \caption{DeepMind Lab} \label{fig:maze-sample}
    \includegraphics[width=\linewidth]{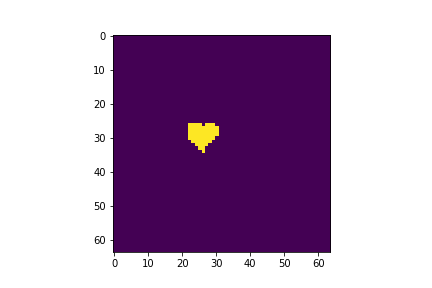}
    \vspace{-20pt}
    \caption{dSprite} \label{fig:dSprite-sample}
  \end{minipage}
\end{figure*}

\cref{fig:kbr_res} summarizes the MSEs over 30 runs when the conditioning points $\tilde x$ are sampled from $\mathcal{N}(\bm{0}_d, V_{XX})$. Here, \texttt{Original} denotes the performance of the original KBR estimator using operator $\hatregorgEQ$ in \eqref{eq:original-kbr}, and \texttt{IW} is computed from IW-KBR approach, which  uses operator $\hatregEQ$ in \eqref{eq:proposed-kbr}. In this experiment, we set $\eta = \lambda = 0.2$ and used Gaussian kernels for both KBR methods, where the bandwidth is given by the median trick. Unsurprisingly, the error increases as  dimension increases. However, the IW-KBR estimator performs significantly better than the original KBR estimator. This illustrates the robustness of the IW-KBR approach even when the model is misspecified, since the correct $\expect[Q]{Z|X}$, which is a linear function of $X$, does not belong to the Gaussian RKHS.

To show how the quality of the density ratio estimate relates to the overall performance, we include the result of IWKBR with the true density ratio $r_0$, denoted as ``IW-TrueRatio'' in \cref{fig:kbr_res}. Surprisingly, IW-TrueRatio performs worse than the original IWKBR, which estimates density ratio using KuLSIF. This suggests that true density ratio yields a suboptimal bias/variance tradeoff, and greater variance for the {\em finite sample estimate of the KBR posterior}, compared with the smoothed estimate obtained from KuLSIF. 

\subsection{Low-Dimensional KBF}
We consider a filtering problem introduced by \citet{Fukumizu2013KBR}, with latent $Z_t \in \mathbb{R}^2$ and observation $X_t \in \mathbb{R}^2$. The dynamics of the latent $Z_t = (u_t, v_t)^\top$ are given as follows. Let $\theta_t \in [0, 2\pi]$ be 
\begin{align*}
    \cos \theta_t = \frac{u_t}{\sqrt{u_t^2 + v_t^2}}, \quad \sin \theta_t = \frac{v_t}{\sqrt{u_t^2 + v_t^2}}.
\end{align*}
The latent $Z_{t+1}= (u_{t+1}, v_{t+1})^\top$ is then
\begin{align}
    \begin{pmatrix}
        u_{t+1}\\
        v_{t+1}
    \end{pmatrix} = (1 + \beta\sin(M\theta_t))\begin{pmatrix}
        \cos (\theta_t + \omega)\\
        \sin (\theta_t + \omega)
    \end{pmatrix} + \epsilon_Z, \label{eq:dynamics}
\end{align}
for given parameters $(\beta, M, \omega)$. The observation $X_t$ is given by $X_t = Z_t + \epsilon_X$.
Here, $\epsilon_X$ and $\epsilon_Z$ are noise variables sampled from $\epsilon_X \sim \mathcal{N}(0, \sigma^2_X I)$ and $\epsilon_Z \sim \mathcal{N}(0, \sigma^2_Z I)$. 

\citeauthor{Fukumizu2013KBR} test performance using the ``Rotation'' dynamics $\omega = 0.3, b = 0$ and ``Oscillatory'' dynamics $\omega = 0.4, b = 0.4, M = 8$, where the noise level is set to $\sigma_X = \sigma_Z = 0.2$ in both scenarios. Using the same dynamics, we evaluate the performance of our proposed estimator by the MSE in predicting $Z_t$ from $X_{1:t}$, where the length of the test sequence is set to 200. We repeated the experiments  30 times for each setting. 

Results are summarized in \cref{fig:low-dimensional}:
 \texttt{Original} denotes the results for using original KBR approach in \cref{alg:kbf}, while \texttt{IW} used our IW-KBR approach. For both approaches, we used Gaussian kernels  $k_\calX, k_\calZ$ whose bandwidths are set to $\beta D_\calX$ and $\beta D_\calZ$, respectively. Here, $D_\calX$ and $D_\calZ$ are the medians of pairwise distances among the training samples. We used the  KuLSIF leave-one-out cross-validation procedure  \citep{Kanamori2012KuLSIF} to tune the regularization parameter $\eta,$ and set $\lambda' = \eta$. This leaves two parameters to be tuned: the scaling parameter $\beta$ and the regularization parameter $\lambda$.  These are selected  using the last 200 steps of the training sequence as a validation set. We also include the results for the extended Kalman filter (EKF) and the particle filter \citep{sarkka2013bayesian}.

\cref{fig:low-dimensional} shows that the EKF and the particle filter perform slightly better than KBF methods in ``Rotation'' dynamics, which replicates the results in \citeauthor{Fukumizu2013KBR}. This is not surprising, since these methods have access to the true dynamics, which makes the tracking easier. In ``Oscillatory'' dynamics, however, which has a stronger nonlinearity, KBF displays comparable or better performance than the EKF, which suffers from a large error caused by the linear approximation. In both scenarios, IW-KBR slightly outperforms the original KBR.

\subsection{High-Dimensional KBF}
Finally, we apply KBF to scenarios where observations are given by  images. We set up two experiments: one uses high dimensional complex images while the latent follows simple dynamics, while the other considers  complex dynamics with observations given by relatively simple images. In both cases, the  particle filter performs significantly worse than the methods based on neural networks, since  likelihood evaluation is unstable in the high dimensional observation space, despite the true dynamics being available.

{\bf Deepmind Lab Video:} The first high-dimensional KBF experiment uses DeepMind Lab \citep{DeepMindLab}, which is a 3D game-like platform for reinforcement learning. This platform provides a first-person viewpoint through the eyes of the agent. An example image can be found in \cref{fig:maze-sample}. Based on this platform, we introduce a filtering problem to estimate the agent's orientation at a specific point in the maze.

The dynamics are  as follows.  Let $\theta_t \in [0, 2\pi)$ be the direction that an agent is facing at time $t$. The next direction is
\begin{align*}
    \theta_{t+1} =  \theta_t + \omega + \epsilon_\theta \quad (\mathrm{mod}~2\pi),
\end{align*}
where noise $\epsilon_\theta \sim \mathcal{N}(0, \sigma^2_\theta)$. Let $Z_t = (\cos \theta_t, \sin \theta_t)$ and $X_t$ be the image that corresponds to the direction $\theta_t + \epsilon_X$ where $\epsilon_X \sim \mathcal{N}(0, \sigma^2_\calX)$. We used the RGB images down-scaled to $(36, 48)$, which makes the observation $36 \times 48 \times 3 = 5184$ dimensions. We  added Gaussian noise $\mathcal{N}(0, 0.01)$ to the each dimension of the observations $Z_t$.  We ran the experiments with $\omega = 0.4, \sigma_\theta = 0.2, \sigma_X = 0.2$, MSEs in 30 runs are summarized in \cref{fig:maze-result}.

\begin{figure}[ht]
    \centering
    \includegraphics[width=\columnwidth]{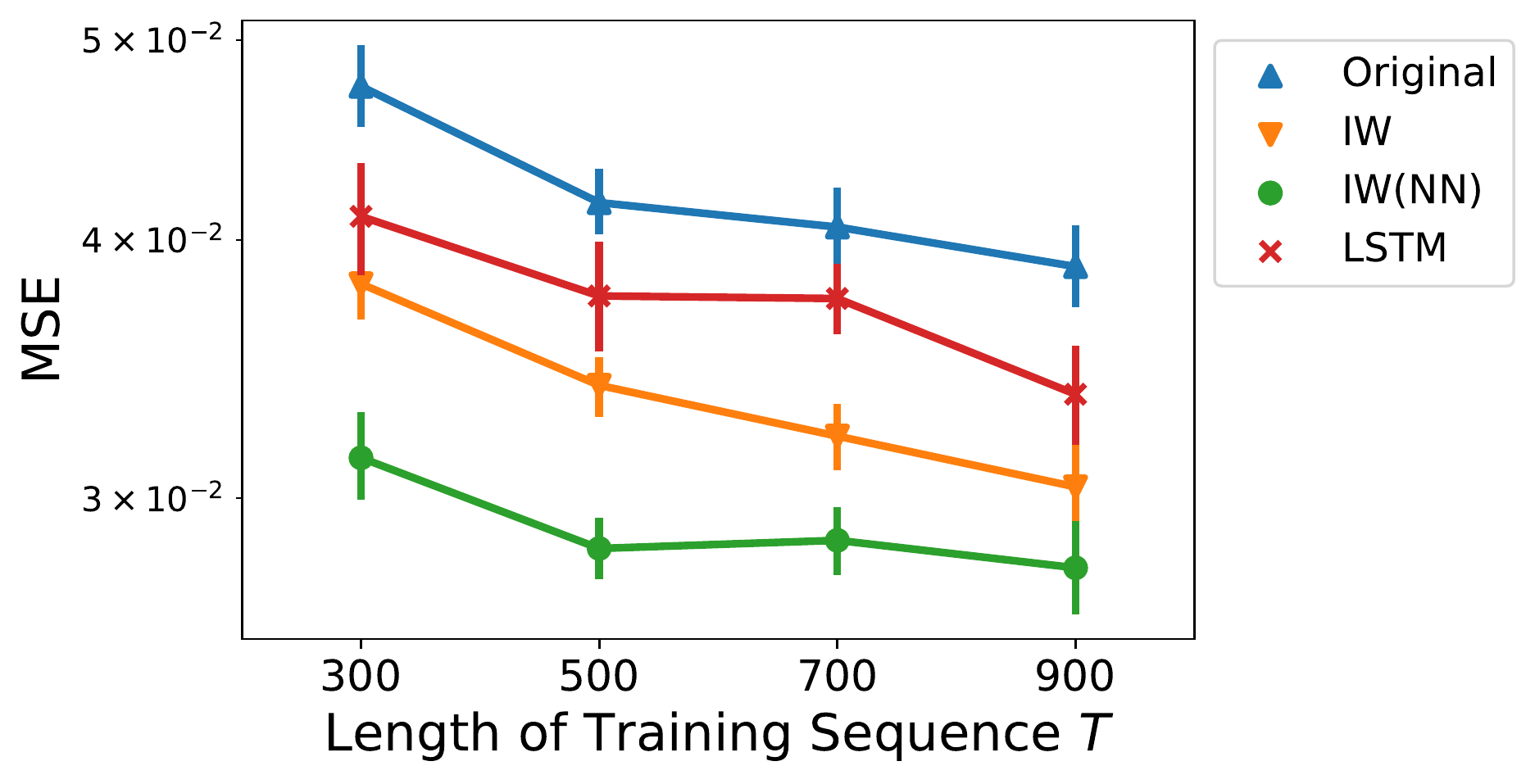}
    \caption{Result for DeepMind Lab setting}
    \label{fig:maze-result}
\end{figure}

In addition to \texttt{Original} and \texttt{IW}, whose hyper-parameters are selected by the same procedure as in the low-dimensional experiments, we introduce two neural network based methods: \texttt{IW(NN)} and \texttt{LSTM}. \texttt{IW(NN)} uses adaptive feature discussed in \cref{subsec:adaptive-feature}. Here, instead of learning different features for each timestep, which would be  time consuming and redundant, we learn adaptive feature $\psi_\theta$ by minimizing 
\begin{align*}
    \hatlossP(E, \theta) = \frac1n \sum_{i=1}^n \|\tilde\phi(z_i) - E^*\psi_\theta(x_i)\|^2 + \lambda \|E\|^2
\end{align*}
with the learned parameter $\hat{\theta}$ being used for all timesteps.  \texttt{LSTM} denotes an LSTM baseline \citep{HochSchm97LSTM}, which predicts $Z_t$ from features extracted from $X_{t-10:t}$. To make the comparison fair,  \texttt{LSTM} used the same network architecture in the feature extractor as \texttt{IW(NN)}.

As in low-dimensional cases, \texttt{IW} consistently performs better for all sequence lengths
 (\cref{fig:maze-result}). The baseline \texttt{LSTM}  performs similarly to \texttt{IW}, however, even though it does not explicitly model the dynamics. This is because functions in the RKHS are not expressive enough to model the relationship between the direction and the images. This is mitigated by using adaptive features in \texttt{IW(NN)}, which outperforms other methods by taking the advantage of the strong expressive power of neural networks.

{\bf dSprite Setting:} The second high-dimensional KBF experiment uses dSprite \citep{dsprites}, which is a dataset of 2D shape images as shown in \cref{fig:dSprite-sample}. Here, we consider the latent $Z_t$ following the same dynamics as \eqref{eq:dynamics}, and the model observes the image $X_t$ where the position of the shape corresponds to the noisy observation of the latent $Z_t + \varepsilon_X, \varepsilon_X \sim \mathcal{N}(0, \sigma^2_X I)$. We used the ``Oscillatory'' dynamics (i.e. $\omega = 0.4, b = 0.4, M = 8$) and set noise levels to $\sigma_X = 0.05, \sigma_Z = 0.2$. Again, we added Gaussian noise $\mathcal{N}(0, 0.01)$ to the images.

\begin{figure}[ht]
    \centering
    \includegraphics[width =\columnwidth]{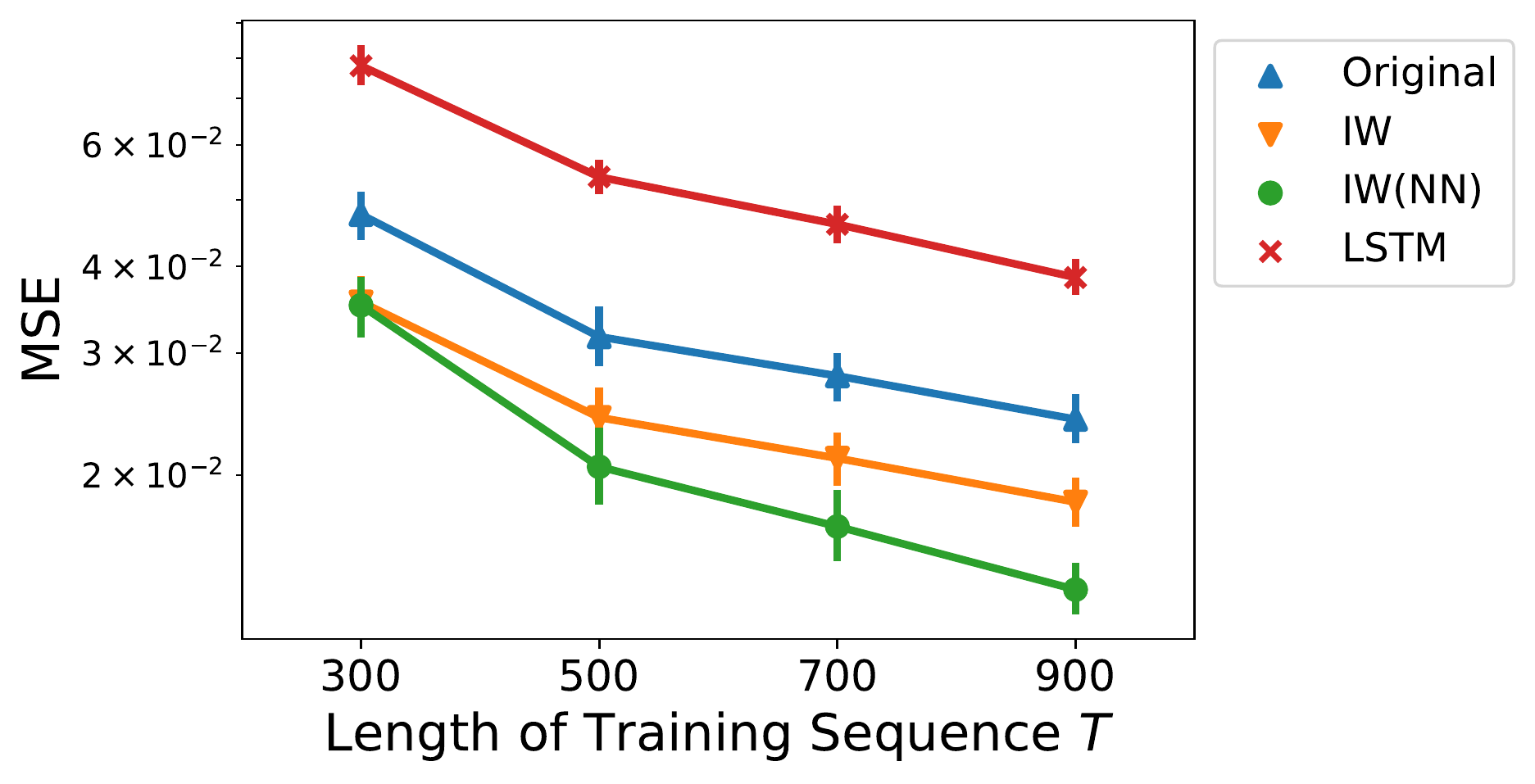}\caption{Result for dSprite setting}
    \label{fig:dsprite-result} 
\end{figure}

MSEs across 30 runs are summarized in \cref{fig:dsprite-result}. Unlike in DeepMind Lab setting, \texttt{LSTM} performs the worst in this setting. This suggests the advantage of KBF methods, which explicitly model the dynamics and exploit them in the filtering. Among KBF methods, \texttt{IW} and \texttt{IW(NN)} perform significantly better than \texttt{Original}, demonstrating the superiority of the IW approach.

\section{Conclusion}
In this paper, we proposed a novel approach to kernel Bayes' rule,  IW-KBR, which minimizes a loss estimated by importance weighting. We established consistency of IW-KBR based on a novel analysis of an RKHS density ratio estimate, which is of independent interest. Our empirical evaluation shows that the IW-KBR significantly outperforms the existing estimator in both Bayesian inference and  filtering for state-space models. Furthermore, by learning adaptive neural net features, IW-KBR outperforms a neural sequence model in  filtering problems with high-dimensional image observations.

In future work, we suggest exploring different density ratio estimation techniques for our setting. It is well-known in the density ratio estimation context that KuLSIF estimator may suffer from high variance. To mitigate this, \citet{Yamada2011RDRE} proposed to use relative density-ratio estimation. Deriving a consistency result for such an estimator and applying it in kernel Bayes' rule would be a promising approach. It would further be of  interest to apply IW-KBR in additional settings, such as approximate Bayesian computation \citep{ABC1,ABC2}, as also discussed by \citet{Fukumizu2013KBR}.

\section*{Acknowledgements}
This work was supported by the Gatsby Charitable Foundation. We thank Dr. Jiaxin Shi for having an informative discussion and suggesting related work.

\bibliographystyle{styles/icml2022}
\bibliography{reference}

\newpage
\appendix
\onecolumn
\section{Implication of \cref{assum:density-ratio}} \label{sec:dre-assume}

In this appendix, we will discuss the implication of \cref{assum:density-ratio} when we consider the ratio of two Gaussian distributions.  Let $\pi \sim \mathcal{N}(0, 1)$ and $p \sim \mathcal{N}(0, 2)$. Then, the density ratio is 
\begin{align*}
 r_0(z) = \frac{\pi(z)}{p(z)} = \sqrt{2} \exp\left(-\frac{z^2}{4}\right),   
\end{align*}
which is in the RKHS $\mathcal{H}_\mathcal{Z}$ induced by Gaussian kernel $k(z, z') = \exp(-(z-z')^2/4)$. Indeed, we can see that
\begin{align*}
    r_0(z) = \sqrt{2}k(z,0). 
\end{align*}
Note that, from reproducing property, we have $\langle r_0, f \rangle = \sqrt{2}f(0)$ for all $f\in\mathcal{H}_\mathcal{Z}$.

Given this, we can analytically compute the eigendecomposition of the covariance operator as 
\begin{align*}
    C^{(P)}_{ZZ} = \sum_i \lambda_i (\sqrt{\lambda_i}e_i(z)) \otimes (\sqrt{\lambda_i} e_i(z)),
\end{align*}
where $\lambda_i = O(B^k)$ with constant $B<1$ and $e_i$ is the $i$-th order Hermite polynomial \citep[][Section 4.3]{RasWil06}. Assumption 3.2 requires $\|(C^{(P)}_{ZZ})^{-\beta_1} r_0\|^2$  finite, meaning $\beta_1$ is the maximum value that satisfies
\begin{align*}
    \|(C^{(P)}_{ZZ})^{-\beta_1} r_0\|^2 = \sum_i \lambda_i^{-2\beta_1} \langle \sqrt{\lambda_i} e_i, r_0 \rangle^2 = 2 \sum_i \lambda_i^{1 -2\beta_1} e^2_i(0) \leq \infty.
\end{align*}

\section{Proof of Theorems} \label{sec:prof}
In this appendix, we provide the proof of our theorems. 

\subsection{Proof for Convergence Analysis} \label{subsec:proof}
We will rely on the following concentration inequality.
\begin{proposition}[Lemma 2 of \citep{Smale2007LT}] \label{prop:concentration}
    Let $\xi$ be a random variable taking values in a real separable RKHS $\calH$ with $\|\xi\|\leq M$ almost surely, and let $\xi_1,\dots, \xi_n$ be i.i.d. random variables distributed as $\xi$. Then, for all $n \in \mathbb{N}$ and $\delta \in (0,1)$,
    \begin{align*}
        \prob{\left\|\frac1n \sum_{i=1}^n \xi_i - \expect{\xi}\right\|\leq \frac{2M\log 2/\delta}{n} + \sqrt{\frac{2\expect{\|\xi\|^2}\log 2/\delta}{n}}} \geq 1 - \delta.
    \end{align*}
\end{proposition}

\subsubsection{Proof of \cref{thm:stage1-converge,thm:stage1-converge-operator}}
We review some properties of the density ratio $r_0$ used in KuLSIF.
\begin{lemma}[\citep{Kanamori2009LSIF}] \label{lem:lsif}
    If the density ratio $r_0  \in \calH_\calZ$, then we have 
    \begin{align*}
        \covPZZ r_0 = m_\Pi,
    \end{align*}
    where $m_\Pi = \expect[\Pi]{\phi(Z)}$.
\end{lemma}
\begin{proof}
From reproducing characteristics, we have 
\begin{align*}
    \covPZZ r_0 = \expect[P]{r_0(Z)\phi(Z)} = \expect[\Pi]{\phi(Z)}.
\end{align*}
The last equality holds from the definition of the density ratio $r_0(z) = \pi(z)/p(z)$.
\end{proof}

Given Lemma \ref{lem:lsif}, we can bound the error of untruncated KuLISF estimator in RKHS norm. Let $\tilde{r}_\eta$ be \begin{align*}
    \tilde{r}_\eta = \left(\hatcovPZZ + \eta I\right)^{-1}\hat{m}_\Pi.
\end{align*}
Note that the weight $\gamma_i$ appearing in the original KBR \eqref{eq:original-gamma} can be understood as the value of $\tilde{r}_\eta$ at specific data points: $\gamma_i = \tilde{r}_\eta(z_i)$. Hence, the weight we use can be written as
\begin{align*}
    \hat{r}_i = \max(0, \tilde{r}_\eta(z_i)).
\end{align*}
Although we use truncated weights as above in KBR, we first derive an upper-bound on the error $\|r_0 - \tilde{r}_\eta\|$. 
Furthermore, we define the function $r_\eta$ which is a popular version of $\tilde{r}_\eta$. 
\begin{align*}
    r_\eta = \left(\covPZZ + \eta I\right)^{-1}m_\Pi.
\end{align*}
Using $r_\eta$, we can decompose the error as 
\begin{align*}
    \|\tilde{r}_\eta - r_0\| \leq  \|\tilde{r}_\eta - r_\eta\| + \|r_\eta - r_0\|.
\end{align*}
The second term can be bounded using an approach similar to Theorem 6 in \citet{Singh2019KIV}.
\begin{lemma}\label{lem:stage1-bias}
    Suppose \cref{assum:density-ratio}, then we have 
    \begin{align*}
        \|r_\eta - r_0\| \leq \eta^{\beta_1}\zeta_1.
    \end{align*}
\end{lemma}
\begin{proof}
        Let $(\nu_k, e_k)$ be the eigenvalues and eigenfunctions of $\covPZZ$, then 
    \begin{align*}
        \|r_\eta - r_0\|^2 &= \| ((\covPZZ + \eta I)^{-1} \covPZZ - I)r_0\|^2\\
        & = \sum_{k} \left(\frac{\nu_k}{\nu_k + \eta} - 1\right)^2 \braket{r_0, e_k}^2\\
        & = \sum_{k} \left(\frac{\eta}{\nu_k + \eta}\right) ^2\braket{r_0, e_k}^2 \left(\frac{\eta}{\eta} \cdot  \frac{\nu_k}{\nu_k} \cdot \frac{\eta + \nu_k}{\eta+ \nu_k}\right)^{2\beta_1}\\
        & = \eta^{2\beta_1} \sum_{k} \nu^{-2\beta_1}_k\braket{r_0, e_k}^2 \left(\frac{\eta}{\nu_k + \eta}\right)^{2-2\beta_1} \left(\frac{\nu_k}{\nu_k + \eta}\right)^{2\beta_1}\\
        & \leq \eta^{2\beta_1} \sum_{k} \nu^{-2\beta_1}_k\braket{r_0, e_k}^2  =\eta^{2\beta_1} \zeta_1^2,
    \end{align*}
    where the first equality uses \cref{lem:lsif}, and the last equality holds from \cref{assum:density-ratio} and the fact that
    \begin{align*}
        \|g_1\|^2 =  \|(\covPZZ)^{-\beta_1} r_0\|^2 = \sum_{k} \nu^{-2\beta_1}_k\braket{r_0, e_k}^2.
    \end{align*}
    \end{proof}
We now establish a useful lemma on the norm of functions.
\begin{lemma}\label{lem:stage1-norm}
     $\|r_\eta\| \leq \|r_0\|$.
\end{lemma}
\begin{proof}
    Let $(\nu_k, e_k)$ be the eigenvalues and eigenfunctions of $\covPZZ$. From \cref{lem:lsif}, we have
    \begin{align*}
        \|r_\eta\|^2 &= \| ((\covPZZ + \eta I)^{-1} \covPZZ)r_0\|^2\\
        & = \sum_{k} \left(\frac{\nu_k}{\nu_k + \eta}\right)^2 \braket{r_0, e_k}^2\\
        & \leq \sum_{k} \braket{r_0, e_k}^2 = \|r_0\|^2.
    \end{align*}
\end{proof}
Now, we are ready to prove our convergence result on $\tilde{r}_\eta$.
\begin{theorem}\label{thm:stage1-rkhs-converge}
    Suppose \cref{assum:density-ratio,assum:bound-feature-norm}. Given data $\{x_i, z_i\} \sim P$ and the estimated prior embedding $\hat{m}_\Pi$, we have 
\begin{align*}
    \|\tilde{r}_\eta - r_0 \| \leq \frac{1}\eta \left(\frac{4\kappa^2_1\|r_0\|\log 2/\delta}{\sqrt{n}} + \|\hat{m}_\Pi - \expect[\Pi]{\phi(X)}\|\right) + \eta^{\beta_1}\zeta_1
\end{align*}
with probability at least $1-\delta$ for $\delta \in (0, 2/\Ne)$.
\end{theorem}
\begin{proof}
    We decompose the error as 
    \begin{align*}
         \|\tilde{r}_\eta - r_0\| \leq  \|\tilde{r}_\eta - r_\eta\| + \|r_\eta - r_0\|.
    \end{align*}
    From \cref{lem:stage1-bias}, we have $\|r_\eta - r_0\| \leq \eta^{\beta_1}\zeta_1$. For the first term, we have
    \begin{align*}
        \left\| r_\eta - \tilde{r}_\eta \right\| &= \|(\hatcovPZZ + \eta I)^{-1}(\hatcovPZZ r_\eta + \eta r_\eta - \hat{m}_\Pi)\|\\
        & \leq \|(\hatcovPZZ + \eta I)^{-1}\|\|\hatcovPZZ r_\eta + (m_\Pi - \covPZZ r_\eta) - \hat{m}_\Pi)\|\\
        &\leq \frac1\eta \left(\|(\hatcovPZZ - \covPZZ)r_\eta\| +\|\hat{m}_\Pi - \expect[\Pi]{\phi(X)}\|\right).
    \end{align*}
    By applying \cref{prop:concentration} with $\xi = (\phi(Z) \otimes \phi(Z))r_\eta$, we have
    \begin{align*}
        \|(\hatcovPZZ - \covPZZ)r_\eta\| \leq \frac{2\kappa^2_1\|r_0\|\log(2/\delta)}{n} + \sqrt{\frac{2\kappa^4_1\|r_0\|^2\log(2/\delta)}{n}} 
    \end{align*}
    with probability $1-\delta$ since
    \begin{align*}
        &\|\xi\| \leq \|\phi(Z) \otimes \phi(Z)\|\|r_\eta\| \leq \kappa_1^2 \|r_\eta\| \leq \kappa_1^2 \|r_0\|,\\
        &\expect{\|\xi\|^2} \leq \kappa^4_1 \|r_\eta\|^2 \leq \kappa^4_1 \|r_0\|^2
    \end{align*}
    from \cref{lem:stage1-norm}. Since $2\log(2/\delta) \geq \sqrt{2}\log(2/\delta)$ for $\delta \in (0, 2/\Ne)$, we finally obtain 
    \begin{align*}
    \|\tilde{r}_\eta - r_0 \| \leq \frac{1}\eta \left(\frac{4\kappa^2_1\|r_0\|\log 2/\delta}{\sqrt{n}} + \|\hat{m}_\Pi - \expect[\Pi]{\phi(X)}\|\right) + \eta^{\beta_1}\zeta_1.
\end{align*}
\end{proof}

Given \cref{thm:stage1-rkhs-converge}, \cref{thm:stage1-converge} is now easy to prove.
\begin{proof}[Proof of \cref{thm:stage1-converge}]
    Note that we have
    \begin{align*}
        \|\hat{r} - r_0(z)\|_\infty & = \max_{z\in\calZ} |\max(\tilde{r}_\eta(z), 0) - r_0(z)|\\
        & \leq \max_{z\in\calZ} |\tilde{r}_\eta(z) - r_0(z)|\\
        &\leq \|\tilde{r}_\eta- r_0\| \max_{z\in\calZ} \|\phi(z)\| \\
        &\leq \kappa_1\|\tilde{r}_\eta- r_0\|,
    \end{align*}
    where the first inequality uses the fact that density ratio $r_0$ is non-negative. From the upper-bound in \cref{thm:stage1-rkhs-converge}, we obtain
    \begin{align*}
        \|\hat{r} - r_0(z)\|_\infty &\leq \kappa_1\|\tilde{r}_\eta- r_0\|\\
        &\leq \frac{\kappa}\eta \left(\frac{4\kappa^2_1\|r_0\|\log 2/\delta}{\sqrt{n}} +\|\hat{m}_\Pi - \expect[\Pi]{\phi(X)}\|\right) + \eta^{\beta_1}\kappa_1\zeta_1.
    \end{align*}
    with probability $1-\delta$ for $\delta \in (0, 2/\Ne)$. Hence, if $\|\hat{m}_\Pi - \expect[\Pi]{\phi(X)}\|\leq O_p(n^{-\alpha_1})$, by setting $\eta = O(n^{-\frac{\alpha_1}{\beta_1 + 1}})$, we have 
    \begin{align*}
         \|\hat{r} - r_0(z)\|_\infty \leq O_p(n^{-\frac{\alpha_1\beta_1}{\beta_1 + 1}}).
    \end{align*}
\end{proof}

Furthermore, we can show the following convergence result on the covariance operator as follows.
\begin{proof}[Proof of \cref{thm:stage1-converge-operator}]
    From \cref{thm:stage1-converge}, with probability $1-\delta$ for $\delta \in (0, 2/\Ne)$, we have 
    \begin{align*}
        \max_i |r_i - r_0(z_i)|\leq \|\hat{r} - r_0\|_\infty 
        &\leq \frac{\kappa_1}\eta \left(\frac{4\kappa^2_1\|r_0\|\log 2/\delta}{\sqrt{n}} + \|\hat{m}_\Pi - \expect[\Pi]{\phi(X)}\|\right) + \eta^{\beta_1}\kappa_1\zeta_1.
    \end{align*}
    Let $\tilcovQXX$ be the empirical covariance operator with true density ratio $r_0$:
    \begin{align*}
        \tilcovQXX=\sum_{i=1}^n r_0(z_i)\psi(x_i)\otimes\psi(x_i).
    \end{align*}
    Then, we have
    \begin{align*}
            \|\hatcovQXX - \covQXX\| \leq \|\hatcovQXX - \tilcovQXX\| + \|\covQXX - \tilcovQXX\|.
    \end{align*}
    For the first term, we have 
    \begin{align*}
        \|\hatcovQXX - \tilcovQXX\| &\leq \frac1n \sum_i |r_i - r_0(z_i)|\|\psi(x_i) \otimes \psi(x_i)\|\\
        &\leq \frac{\kappa_1\kappa^2_2}\eta \left(\frac{4\kappa^2_1\|r_0\|\log 2/\delta}{\sqrt{n}} + \|\hat{m}_\Pi - \expect[\Pi]{\phi(X)}\|\right) + \eta^{\beta_1}\kappa_1\kappa^2_2\zeta_1.
    \end{align*}
    For the second term, by applying \cref{prop:concentration} with $\xi = r_0(Z)(\psi(X) \otimes \psi(X))$, we have
    \begin{align*}
        \|\covQXX - \tilcovQXX\| &\leq \frac{2 \kappa_1 \kappa_2^2\|r_0\|\log 2/\delta}{n} + \sqrt{\frac{2(\kappa_1 \kappa_2^2\|r_0\|)^2\log 2/\delta}{n}}\\
        &\leq \frac{4 \kappa_1 \kappa_2^2\|r_0\|\log 2/\delta}{\sqrt{n}}
    \end{align*}
    with probability $1-\delta$ for $\delta \in (0, 2/\Ne)$ since
    \begin{alignat*}{3}
        &\|\xi\| &&= |r_0|\|\psi(X) \otimes \psi(X)\|\\
        &&&\leq \kappa_1 \kappa_2^2\|r_0\|,\\
        &\expect{\|\xi\|^2} &&\leq (\kappa_1 \kappa_2^2\|r_0\|)^2.
    \end{alignat*}
    Hence, we have 
    \begin{align*}
        \|\hatcovQXX - \covQXX\| &\leq \|\hatcovQXX - \tilcovQXX\| + \|\covQXX - \tilcovQXX\|\\
        &\leq \frac{\kappa_1\kappa^2_2}\eta \left(\frac{4\kappa^2_1\|r_0\|\log 2/\delta}{\sqrt{n}} + \|\hat{m}_\Pi - \expect[\Pi]{\phi(X)}\|\right) + \eta^{\beta_1}\kappa_1\kappa^2_2\zeta_1 + \frac{4 \kappa_1 \kappa_2^2\|r_0\|\log 2/\delta}{\sqrt{n}}
    \end{align*}
    with probability $1-2\delta$ for $\delta \in (0, 2/\Ne)$. Thus, if $\|\hat{m}_\Pi - \expect[\Pi]{\phi(X)}\| \leq O_p(n^{-\alpha_1})$, by setting $\eta = O(n^{-\frac{\alpha_1}{\beta_1 + 1}})$, we have 
    \begin{align*}
         \|\hatcovQXX - \covQXX\| \leq O_p(n^{-\frac{\alpha_1\beta_1}{\beta_1 + 1}}).
    \end{align*}
\end{proof}

\subsubsection{Proof of \cref{thm:stage2-converge}}
Next, we derive \cref{thm:stage2-converge}. Again, we define the operator $\regEQ$ which replaces the empirical estimates in $\hatregEQ$ by their population versions.
\begin{align*}
    \regEQ = (\covQXX + \lambda I)^{-1}\covQXZ
\end{align*}
By following a similar approach as in \cref{lem:stage1-bias}, we obtain the following result.
\begin{lemma}[Theorem 6 in \citep{Singh2019KIV}]\label{lem:stage2-bias}
    Suppose \cref{assum:covariance-operator}, then we have 
    \begin{align*}
        \|\regEQ - \EQ\| \leq \lambda^{\beta_2}\zeta_2.
    \end{align*}
\end{lemma}
Using a proof similar to the one of \cref{lem:stage1-norm}, we also have
\begin{lemma}\label{lem:stage2-norm}
    $\|\regEQ\| \leq \|\EQ\|$.
\end{lemma}

Given these lemmas, we can prove \cref{thm:stage2-converge} as follows. 
\begin{proof}[Proof of \cref{thm:stage2-converge}]
    We decompose the error as follows:
    \begin{align*}
        \|\EQ - \hatregEQ\| \leq \|\EQ - \regEQ\| + \|\regEQ - \hatregEQ\|.
    \end{align*}
    \cref{lem:stage2-bias} bounds the first term as 
     \begin{align*}
        \|\regEQ - \EQ\| \leq \lambda^{\beta_2}\zeta_2.
    \end{align*}
    The second term can be bounded as follows 
    \begin{align*}
       \|\regEQ - \hatregEQ\|  &= \| (\hatcovQXX + {\lambda}I )^{-1}(\hatcovQXX \regEQ + {\lambda} \regEQ - \hatcovQXZ)\|\\
    & \leq \| (\hatcovQXX + {\lambda} I )^{-1}\|\|(\hatcovQXX \regEQ + {\lambda} \regEQ - \hatcovQXZ)\| \\
    & \leq \frac1\lambda \|(\hatcovQXX \regEQ + {\lambda} \regEQ - \hatcovQXZ)\| \\
    & = \frac1\lambda \|(\hatcovQXX \regEQ + (\covQXZ - \covQXX\regEQ) - \hatcovQXZ)\|\\
    & = \frac1\lambda \|(\hatcovQXX- \covQXX) \regEQ - (\hatcovQXZ - \covQXZ)\|\\
    & = \frac1\lambda (\|\hatcovQXX- \covQXX)\|\|\regEQ\| + \|\hatcovQXZ- \covQXZ\|)\\
    & = \frac1\lambda (\|\hatcovQXX- \covQXX)\|\|\EQ\| + \|\hatcovQXZ- \covQXZ\|)
    \end{align*}
    where the last inequality uses \cref{lem:stage2-norm}. Therefore, we have 
    \begin{align*}
        \|\EQ - \hatregEQ\| \leq \frac1\lambda (\|\hatcovQXX- \covQXX\|\|\EQ\| + \|\hatcovQXZ- \covQXZ\|) + \lambda^{\beta_2}\zeta_2.
    \end{align*}
    Hence, if $\|\hatcovQXX- \covQXX\| \leq O_p(n^{-\alpha_2})$ and $\|\hatcovQXZ- \covQXZ\| \leq O_p(n^{-\alpha_2})$ , we can see 
    \begin{align*}
        \|\EQ - \hatregEQ\| \leq O_p(n^{-\frac{\alpha_2\beta_2}{\beta_2 + 1}})
    \end{align*}
    by setting $\lambda = O(n^{-\frac{\alpha_2}{\beta_2 + 1}})$.
\end{proof}

\subsection{Proof for Adaptive Features} \label{subsec:proof-adaptive-feature}
Here, we derive the loss function $\ell(\theta)$ for adaptive feature. Let $\Phi = [\phi(z_1), \dots, \phi(z_n)]$ be the feature map for $Z$. Then the loss $\hatlossQ$ can be written as 
\begin{align*}
    \hatlossQ(E, \theta) &= \mathrm{tr}\left(\left(\Phi - E^*\Psi_\theta\right) D\left(\Phi - E^*\Psi_\theta\right)^\top \right) + \lambda \|E\|^2\\
    &= \mathrm{tr}\left(\Phi D \Phi^\top - 2E^*\Psi_\theta D \Phi^\top + E^*\Psi_\theta D \Psi_\theta ^\top E\right) + \lambda \|E\|^2.
\end{align*}
Since for fixed $\theta$, the minimizer of $\hatlossQ$ with respect to $E$ can be written as 
\begin{align*}
    \hatregEQ =  (\Psi_\theta D \Psi_\theta^\top + \lambda I)^{-1} \Psi_\theta D \Phi^\top,
\end{align*}
we have
\begin{align*}
    \hatlossQ(\hatregEQ, \theta) &= \mathrm{tr}\left(\Phi D \Phi^\top \right) -2 \mathrm{tr}\left(\Phi D \Psi_\theta^\top(\Psi_\theta D \Psi_\theta^\top + \lambda I)^{-1}\Psi_\theta D \Phi^\top\right)\\
    &\quad \quad +\mathrm{tr}\left( \Phi D \Psi_\theta^\top(\Psi_\theta D \Psi_\theta^\top + \lambda I)^{-1}\Psi_\theta D \Psi_\theta^\top (\Psi_\theta D \Psi_\theta^\top + \lambda I)^{-1} \Psi_\theta D \Phi^\top\right)\\
    &\quad \quad +\mathrm{tr}\left( \lambda \Phi^\top D \Psi_\theta(\Psi_\theta^\top D \Psi_\theta + \lambda I)^{-1} (\Psi_\theta^\top D \Psi_\theta + \lambda I)^{-1} \Psi_\theta^\top D \Phi\right)\\
    & = \mathrm{tr}\left(\Phi D \Phi^\top \right) - \mathrm{tr}\left(\Phi D \Psi_\theta^\top(\Psi_\theta D \Psi_\theta^\top + \lambda I)^{-1}\Psi_\theta D \Phi^\top\right).
\end{align*}
Using $G_Z = \Phi^\top \Phi$, we have
\begin{align*}
    \ell(\theta) &= \hatlossQ(\hatregEQ, \theta)\\
    &=  \mathrm{tr}\left(G D\right) - \mathrm{tr}\left(\Phi D \Psi_\theta^\top(GD \Psi_\theta^\top + \lambda I)^{-1}\Psi_\theta D \Phi^\top\right)\\
    &=  \mathrm{tr}\left(G D -G D \Psi_\theta^\top(\Psi_\theta D \Psi_\theta^\top + \lambda I)^{-1}\Psi_\theta D \right).
\end{align*}

\end{document}